\documentclass[11pt]{article}

\usepackage[letterpaper,margin=1in]{geometry}
\usepackage{epsfig,color}
\usepackage{amssymb,amsfonts,latexsym}
\usepackage{amsmath,amsthm}
\usepackage{algorithm,algorithmic}
\usepackage{enumerate}
\usepackage{tikz}
\usepackage{graphicx}
\usepackage[tight]{subfigure}

\newcommand{\Nat}{{\mathbb N}}
\renewcommand{\S}{\mathcal S}

\newcommand{\E}{\operatorname{E}}
\newcommand{\ra}{\rightarrow}
\newcommand{\eps}{\epsilon}

\newcommand{\poly}{\operatorname{poly}}

\newcommand{\err}{\operatorname{err}}
\newcommand{\opt}{\operatorname{opt}}

\newtheorem{theorem}{Theorem}
\newtheorem{corollary}{Corollary}

\newtheorem{proposition}{Proposition}
\newtheorem{lemma}{Lemma}

\theoremstyle{definition}

\newcommand{\sgn}{\operatorname{sgn}}
\newcommand{\RW}{\operatorname{RW}}

\newcommand{\T}{\mathcal T}
\newcommand{\C}{\mathcal C}
\newcommand{\R}{\mathbb R}
\renewcommand{\S}{\mathcal S}

\begin{document}

\title{Agnostically Learning Juntas from Random Walks}
\author{
Jan Arpe\thanks{U.C.\ Berkeley. Email: {\tt arpe@stat.berkeley.edu}. Supported by the Postdoc-Program of the German Academic Exchange Service (DAAD) and in part by NSF Career Award DMS 0548249 and BSF 2004105.} 
\quad 
Elchanan Mossel\thanks{U.C.\ Berkeley. Email: {\tt mossel@stat.berkeley.edu}. Supported by NSF Career Award DMS 0548249, BSF 2004105, and DOD ONR grant N0014-07-1-05-06.}
}

\date{June 25, 2008} 

\maketitle

\thispagestyle{empty}

\begin{abstract}
We prove that the class of functions $g:\{-1,+1\}^n\ra\{-1,+1\}$ that only depend on an unknown subset of $k\ll n$ variables (so-called $k$-juntas) is agnostically learnable from a random walk in time polynomial in $n$, $2^{k^2}$, $\eps^{-k}$, and $\log(1/\delta)$. In other words, there is an algorithm with the claimed running time that, given $\eps,\delta>0$ and access to a random walk on $\{-1,+1\}^n$ labeled by an arbitrary function $f:\{-1,+1\}^n\ra\{-1,+1\}$, finds with probability at least $1-\delta$ a $k$-junta that is $(\opt(f)+\eps)$-close to $f$, where $\opt(f)$ denotes the distance of a closest $k$-junta to~$f$.
\end{abstract}


\section{Introduction}\label{sec:intro}

\subsection{Motivation}

In supervised learning, the learner is provided with a {\em training set} of {\em labeled examples} 
\[(x^1,f(x^1)),(x^2,f(x^2)),\ldots\; ,\] 
and the goal is to find a {\em hypothesis}~$h$ that is a good approximation to~$f$, i.e., that gives good estimates for $f(x)$ also on the points that are not present in the training set. In many applications, the points $x$ correspond to particular {\em states} of a system and the labels $f(x)$ correspond to {\em classifications} of these states. If the underlying system evolves over time and thus $(x^t,f(x^t))$ corresponds to a measurement of the current state and its classification at time~$t$, it is often reasonable to assume that state changes only occur {\em locally}, i.e., 
at each time~$t$, $x^t$ differs only ``locally'' from $x^{t-1}$. Such phenomena occur for instance in physics or biology: e.g., in a fixed time interval, a particle can only travel a finite distance 
and the mutation of a DNA sequence can be assumed to happen in a single position at a time. In discrete settings, such processes are often modeled as {\em random walks} on graphs, in which the nodes represent the states of the system, and edges indicate possible local state changes. 

We are interested in studying the special case that the underlying graph is a hypercube, i.e., the node set is $\{-1,1\}^n$ and two nodes are adjacent if and only if they differ in exactly one coordinate. Furthermore, we restrict the setting to {\em Boolean classifications}. 
This {\em random walk learning model} has attracted a lot of attention since the nineties~\cite{AldVaz95AMarkovian,BaFiHo02Exploiting,Gam03Extension,BMOS05Learning,Roc07OnLearning}, mainly because of its interesting learning theoretic properties. The model is weaker than the {\em membership query model} in which the learner is allowed to ask the classifications of specific points, and it is stronger than the {\em uniform-distribution model} in which the learner observes points that are drawn independently of each other from the uniform distribution on $\{-1,1\}^n$. Moreover, the latter relation is known to be {\em strict}: under a standard complexity theoretic assumption (existence of one-way functions) there is a class that is efficiently learnable from labeled random walks, but not from independent uniformly distributed examples~\cite[Proposition~2]{BMOS05Learning}.

The random walk learning model shares some similarities with both other models mentioned above: as in the uniform-distribution model, the examples are generated at random (so that the learner has no influence on the given examples) and points of the random walk that correspond to time points that are sufficiently far apart roughly behave like independent uniformly distributed points. On the other hand, some learning problems that appear to be infeasible in the uniform distribution model but are known to be easy to solve in the membership query model have turned out to be easy in the random walk model as well. Among them is the problem of learning DNFs with polynomially many terms~\cite{BMOS05Learning} (even under random classifciation noise) and the problem of learning parity functions in the presence of random classification noise. The former result relies on an efficient algorithm performing the {\em Bounded Sieve}~\cite{BMOS05Learning} introduced in~\cite{BshFel02OnUsing}. 
The latter result follows from the fact that the (noise-less) random walk model admits an efficient approximation of variable influences, and the effect of {\em random} classification noise can be easily dealt with by drawing a sufficiently larger amount of examples. 

Given this success of the random walk model in learning large classes in the presence of random classification noise, it is natural to ask whether it can also cope with even more severe noise models. One elegant, albeit challenging, noise model is the {\em agnostic learning model} introduced by Kearns et al.~\cite{KeScSe94Toward}. In this model, no assumption whatsoever is made about the labels. Instead of asking for a hypothesis that is close to the classification function, the goal in agnostic learning is to produce a hypothesis that agrees with the labels on nearly as many points as the {\em best} fitting function from the target class. More formally, given a class $\C$ of Boolean functions on $\{-1,1\}^n$ and an arbitrary function $f:\{-1,1\}^n\ra\{-1,1\}$, let $\opt_{\C}(f)=\min_{g\in \C}\Pr[g(x)\neq f(x)]$. The class $\C$ is {\em agnostically learnable} if there is an algorithm that, for any $\eps,\delta>0$, produces a hypothesis~$h$ that, with probability at least $1-\delta$, satisfies $\Pr[h(x)\neq f(x)]\leq \opt_{\C}(f)+\eps$.

Recently, Gopalan et al.~\cite{GoKaKl08Agnostically} have shown that the class of Boolean functions that can be represented by decision trees of polynomial size (in the number of variables) can be learned agnostically from {\em membership queries} in polynomial time. Their main result combines the Kushilevitz-Mansour algorithm for finding large Fourier coefficients~\cite{KusMan93Learning} with a gradient-descent algorithm~\cite{Zin03Online} to solve an $\ell^1$-regression problem for sparse polynomials. They also present a simpler algorithm (with slightly worse running time) that {\em properly} agnostically learns the class of {\em $k$-juntas}. These are functions $f:\{-1,1\}^n\ra\{-1,1\}$ that depend on an a priori unknown subset of at most $k$ variables. The term {\em proper} learning refers to the requirement that only hypotheses from the target class (here: $k$-juntas) are produced.

The investigation of the learnability of this class has both practical and theoretical motivation. Practically, the junta learning problem serves as a clean model of learning in the presence of irrelevant information, a core problem in data mining~\cite{BluLan97Selection}. From a theoretical perspective, the problem is interesting due to its close relationship to learning DNF formulas, decision trees, and noisy parity functions~\cite{MoODSe04Learning}.

\subsection{Our Results and Techniques}

The main result of this paper is that the class of $k$-juntas on $n$ variables is properly agnostically learnable in the random walk model in time polynomial in $n$ (times some function in $k$ and the accuracy parameter $\eps$). More precisely, we show 
\begin{theorem}\label{thm:main}
Let $\C$ be the class of $k$-juntas on $n$ variables. There is an algorithm that, given $\eps,\delta>0$ and access to a random walk $x^1,x^2,\ldots$ on $\{-1,1\}^n$ that is labeled by an arbitrary function $f:\{-1,1\}^n\ra\{-1,1\}$, returns a $k$-junta $h$ that, with probability at least $1-\delta$, satisfies
\[\Pr[h(x)\neq f(x)]\leq \opt_{\C}(f) + \eps 
\; .\] 
The running time of this algorithm is polynomial in $n$, $2^{k^2}$, $(1/\eps)^k$, and $\log(1/\delta)$. 
\end{theorem}

We thus prove the first efficient learning result for agnostically learning juntas (even properly) in a {\em passive} learning model.

Our main technical lemma (Lemma~\ref{lem:spectrum}) shows that for an arbitrary function $f$ and a $k$-junta $g$, there exists another $k$-junta $g'$ that is almost as correlated with $f$ as $g$ is and whose relevant variables can be inferred from all low-level Fourier coefficients of $f$ of a certain size. These Fourier coefficients can in turn be detected using the Bounded Sieve algorithm of Bshouty et al.~\cite{BMOS05Learning} given a random walk labeled by~$f$.
Once a superset~$R$ of the relevant variables of $g'$ is found, it is easy to derive a hypothesis that only depends on at most $k$ variables from $R$ and that best matches the given labels: For each $k$-element subset $J\subsetneq R$, the best matching function with relevant variables in~$J$ is obtained by taking majority votes on points that coincide in these coordinates. Similarly to the classical result of Angluin and Laird~\cite{AngLai88Learning} that a (proper) hypothesis that minimizes the number of disagreements with the labels is close to the target function (in the PAC learning model with random classification noise), we show that such a hypothesis is also a good candidate to satisfy the agnostic learning goal in the random walk model (see Proposition~\ref{prop:consistent}). A similar statement has implicitly been shown in the agnostic PAC learning model (see the proof of Theorem~1 in~\cite{KeScSe94Toward}).

\subsection{Related Work}

Our algorithm for agnostically learning juntas in the random walk
model has some similarities to Gopalan et al.'s recent algorithm for
properly agnostically learning juntas in the membership query
model~\cite{GoKaKl08Agnostically}.
The main differences between the approaches are in two respects:
first, we do not explicitly calculate the quantities $I_{i}^{\leq
k}=\sum_{S:i\in S,|S|\leq k}\hat f(S)^2$ but instead use our technical
lemma mentioned above, which may be of independent interest. Second,
instead of using their characterization of the best fitting junta with
a fixed set of relevant variables in terms of the Fourier spectrum of
$f$ (\cite[Lemma~13]{GoKaKl08Agnostically}), we directly construct
such a best fitting hypothesis by taking majority votes in ambiguous
situations.

Even though we became aware of Gopalan et al.'s result only {\em
after} devising our junta learning algorithm we have decided to adopt
much of their notation to the benefit of the readers.

It should also be noted that a generalization of Gopalan et al.'s
decision tree learning algorithm cannot be adapted for the random walk
model in a straightforward manner: The running time of the only known
analogue of the Kushilevitz-Mansour subroutine for the random walk
model (i.e., the Bounded Sieve) is exponential in the {\em level} up
to which the large Fourier coefficients are sought. In general,
however, sparse polynomials can be concentrated on high levels. It
would be interesting to see if the results
in~\cite{GoKaKl08Agnostically} can also be derived for the restriction
of the class of all $t$-sparse polynomials to $t$-sparse polynomials
of degree roughly $\log(t)$ since for every decision tree of size $t$,
there is an $\eps$-close decision tree of depth $O(\log (t/\eps))$
(cf.~\cite{BshFel02OnUsing}). In this case, the same result should
hold for the random walk model.

\subsection{Organization of This Paper}

We briefly introduce notational and technical prerequisites in Section~\ref{sec:prelim}. The random walk learning model and its agnostic variant are introduced in Section~\ref{sec:models}. Section~\ref{sec:concentration} contains a concentration for random walks and the result on disagreement minimization in the random walk model. The main result on agnostically learning juntas is presented in Section~\ref{sec:juntas}. 
The Appendix contains a formal statement and proof of a result concerning the independence of points in a random walk (Section~\ref{sec:independence}) and an elementary proof of the concentration bound (Section~\ref{sec:proof_concentration}).

\section{Preliminaries}\label{sec:prelim}

Let $\Nat=\{0,1,2,\ldots\}$. For $n\in\Nat$, let $[n]=\{1,\ldots,n\}$.
For $x,x'\in\{-1,1\}^n$, let $x\odot x'$ denote the vector obtained by coordinate-wise multiplication of $x$ and $x'$. For $i\in[n]$, let $e_i$ denote the vector in which all entries are equal to $+1$ except in the $i$th position, where the entry is $-1$. For $f:\{-1,1\}^n\ra\{-1,1\}$, a variable $x_i$ is said to be {\em relevant to $f$} (and $f$ {\em depends} on $x_i$) if there is an $x\in\{-1,1\}^n$ such that $f(x\odot e_i)\neq f(x)$. 
For $i\in [n]$ and $a\in\{-1,1\}$, denote by $f_{x_i=a}:\{-1,1\}^n\ra\{-1,1\}$ the sub-function of $f$ obtained by letting $f_{x_i=a}(x)=f(x')$ with $x'_j=x_j$ if $j\neq i$ and $x'_i=a$.
Thus, $x_i$ is relevant to $f$ if and only if $f_{x_i=1}\neq f_{x_i=-1}$. 
The restriction of a vector $x\in\{-1,1\}^n$ to a subset of coordinates $J\subseteq[n]$ is denoted by $x|_J\in\{-1,1\}^{|J|}$. 
All probabilities and expectations in this paper are taken with respect to the uniform distribution (except when indicated differently).

For $f,g:\{-1,1\}^n\ra\R$, define the inner product 
\[\langle f,g\rangle = \E_x[f(x)g(x)] = 2^{-n}\sum_{x\in\{-1,1\}^n}f(x)g(x)\; .\]
It is well-known that the functions $\chi_S:\{-1,1\}^n\ra\{-1,1\}$, $S\subseteq[n]$, defined by $\chi_S(x)=\prod_{i\in S}x_i$ form an orthonormal basis of the space of real-valued functions on $\{-1,1\}^n$. Thus, every function $f:\{-1,1\}^n\ra\R$ has the unique {\em Fourier expansion}
\[f=\sum_{S\subseteq[n]}\hat f(S)\chi_S\]
where $\hat f(S)=\langle f,\chi_S\rangle$ are the {\em Fourier coefficients} of~$f$. 
Let $\|f\|_2=\langle f,f\rangle^{1/2}=\E[f(x)^2]^{1/2}$. Plancherel's equation states that 
\begin{equation}\label{eqn:plancherel}
\langle f,g\rangle =\sum_{S\subseteq[n]} \hat f(S)\hat g(S), \;
\end{equation}
and from this, Parseval's equation $\|f\|_2^2=\sum_{S\subseteq [n]}\hat f(S)^2$ follows as the special case $f=g$.

For $f,g:\{-1,1\}^n\ra\{-1,1\}$, define the {\em distance} between $f$ and $g$ by
\[\Delta(f,g)=\Pr[f(x)\neq g(x)]\; ,\] 
and for a class $\C=\C_n$ of functions from $\{-1,1\}^n$ to $\{-1,1\}$, let $\opt_{\C}(f)=\min_{g\in \C}\Delta(f,g)$ be the distance of $f$ to a nearest function in $\C$.
It is easily seen that $\Delta(f,g) = (1-\langle f,g\rangle)/2$. Furthermore, for a {\em sample} $\S = (x^i,y^i)_{i=1,\ldots,m}$ with $x^i\in\{-1,1\}^n$ and $y^i\in\{-1,1\}$, let 
\[\Delta(f,\S)=\frac{1}{m}\left|\{i\in\{1,\ldots,m\}\mid f(x^i)\neq y^i\}\right|\]
be the fraction of examples in $\S$ for which the labels disagree with the labeling function~$f$.

\section{The Random Walk Learning Model}\label{sec:models}

\subsection{Learning from Noiseless Examples}

Let $\C=\bigcup_{n\in\Nat} \C_n$ be a class of functions, where each $\C_n$ contains functions $f:\{-1,1\}^n\ra\{-1,1\}$. 
In the {\em random walk learning model}, a learning algorithm has access to the oracle $\RW(f)$ for some unknown function $f\in\C_n$. On the first request, $\RW(f)$ generates a {\em point} $x\in\{-1,1\}^n$ according to the uniform distribution on $\{-1,1\}^n$ and returns the {\em example} $(x,f(x))$, where we refer to $f(x)$ as the {\em label} or the {\em classification} of the example. On subsequent requests, it selects a random coordinate $i\in[n]$ and returns $(x\odot e_i,f(x\odot e_i))$, where $x$ is the point returned in the last query. The goal of a learning algorithm~$\mathcal A$ is, given inputs $\delta,\eps>0$, to output a hypothesis $h:\{-1,1\}^n\ra\{-1,1\}$ such that with probability at least $1-\delta$ (taken over all possible random walks of the requested length), $\Pr[h(x)\neq f(x)]\leq \eps$. In this case, $\mathcal A$ is said to {\em learn}~$f$ with {\em accuracy}~$\eps$ and {\em confidence}~$1-\delta$.

The class $\C$ is {\em learnable from random walks} if there is an algorithm $\mathcal A$ that for every $n$, every $f\in\C_n$, every $\delta>0$, and every $\eps>0$ learns $f$ with access to $\RW(f)$ with accuracy $\eps$ and confidence $1-\delta$.
The class $\C$ is said to be learnable in time equal to the running time of $\mathcal A$, which is a function of $n$, $\eps$, $\delta$, and possibly other parameters involved in the parameterization of the class~$\C$.

If a learning algorithm only outputs hypotheses $h\in\C_n$, it is called a {\em proper learning algorithm}. In this case, $\C$ is {\em properly learnable}. 

The random walk model is a {\em passive} learning model in the sense that a learning algorithm has no direct control on which examples it receives (as opposed to the {\em membership query model} in which the learner is allowed to ask for the labels of specific points~$x$). 
For passive learning models, we may assume without loss of generality that all examples are requested at once.

\subsection{Agnostic Learning}

In the model of {\em agnostic learning from random walks}, we make no assumption whatsoever on the nature of the labels.
Following the model of Gopalan et al.~\cite{GoKaKl08Agnostically}, we assume that there is an {\em arbitrary} function $f:\{-1,1\}^n$ according to which the examples are labeled, i.e., a learner observes pairs $(x,f(x))$, with the points coming from a random walk. In other words, the learner has access to $\RW(f)$, but now $f$ is no longer required to belong to $\C$. We can think of the labels as originating from a concept $g\in\C$, with an $\opt_{\C}(f)$ fraction of labels flipped by an adversary.

The goal of a learning algorithm is to output a hypothesis $h$ that performs nearly as well as the best function of $\C$. Let $\opt_{\C}(f)=\min_{g\in\C}\Pr_x[g(x)\neq f(x)]$, where $x\in\{-1,1\}^n$ is drawn according to the uniform distribution. An algorithm {\em 
agnostically learns $\C$} if, for any $f:\{-1,1\}^n\ra\{-1,1\}$, given $\delta,\eps>0$, it outputs a hypothesis $h:\{-1,1\}^n$ such that with probability at least $1-\delta$, $\Pr_x[h(x)\neq f(x)]\leq\opt_{\C}(f)+\eps$.  
Again, if the algorithm always outputs a hypothesis $h\in\C$, then it is called a {\em proper} learning algorithm, and $\C$ is said to be {\em properly agnostically learnable}.

Although all learning algorithms in this paper are proper, we believe that a word is in order concerning the formulation of the learning goal in {\em improper agnostic learning}. Namely, it could well happen that we can find a hypothesis that satisfies $\Pr_x[h(x)\neq f(x)]\leq \opt_{\C}(f)$, but such an $h$ could be as far as $2\opt_{\C}(f)$ from all concepts in $\C$, which can definitely not be considered a sensible solution if, say, $\opt_{\C}(f)\geq 1/4$. Instead, a hypothesis should rather be required to be $\eps$-close to {\em some} function $g\in \C$ that performs best (or almost best): $\Pr[h(x)\neq g(x)]\leq \eps$ for some $g\in \C$ with $\Pr[g(x)\neq f(x)]=\opt_{\C}(f)$ (or for some near-optimal $g\in\C$ with $\Pr[g(x)\neq f(x)\leq \opt_{\C}(f)+\eps'$). Alternatively, one can require $h$ to belong to some reasonably chosen {\em hypothesis class} $\mathcal H\supseteq \C$, e.g., the hypotheses output by the algorithm in~\cite{GoKaKl08Agnostically} for learning decision trees of size~$t$ are {\em $t$-sparse polynomials}. In fact, that algorithm {\em properly} agnostically learns the latter class.

\section{A Concentration Bound for Labeled Random Walks}\label{sec:concentration}

The following lemma estimates the probability that, after drawing a random walk $x^0,\ldots,x^\ell$, the points $x^0$ and $x^\ell$ are independent. The proof (and a more formal statement) are deferred to the Appendix (see Lemma~\ref{lem:independence_formal} in Section~\ref{sec:independence}).
\begin{lemma}\label{lem:independence}
Let $\delta>0$, $\ell\geq n\ln (n/\delta)$ and $x^0,\ldots,x^\ell$ be a random walk on $\{-1,1\}^n$. Then, with proability at least $1-\delta$, $x^0$ and $x^\ell$ are independent\footnote{More precisely, we can perform an additional experiment such that conditional to some event that occurs with probability at least $1-\delta$ (taken over the draw of the random walk and the outcome of the additional experiment), $x^0$ and $x^\ell$ are independent. For more details, see Section~\ref{sec:independence} in the Appendix.} and uniformly distributed.
\end{lemma}

\begin{lemma}\label{lem:concentration}
Let $g:\{-1,1\}^n\ra[-1,1]$ and $\delta,\eps>0$. Let $N=\lceil n\ln (n/\delta)\rceil$,  
\[m\geq \frac{2N}{\eps^2}\ln\left(\frac{2N}{\delta}\right)\; ,\] 
and $x^1,\ldots,x^m$ be a random walk on $\{-1,1\}^n$.
Then, with probability at least $1-\delta$,
\[\left|\frac{1}{m}\sum_{i=1}^m g(x^i) - \E_x[g(x)]\right|\leq \eps\; ,\]
where the expectation is taken over a uniformly distributed~$x$.
\end{lemma}

Although a similar result can be obtained from the more general works on concentration bounds for random walks by Gillman~\cite{Gil98AChernoff} and for finite Markov Chains by L{\'e}zaud~\cite{Lez98ChernoffType}, we give an elementary proof for Lemma~\ref{lem:concentration} in the Appendix (see Section~\ref{sec:proof_concentration}).

As an immediate consequence, the fraction of disagreements between the labels $f(x^i)$ and the values $h(x^i)$ on a random walk converge quickly to the total fraction of disagreements on all of $\{-1,1\}^n$:
\begin{corollary}\label{cor:convergence_all}
Let $\C=\C_n$ be a class of functions from $\{-1,1\}^n$ to $\{-1,1\}$. Let $\eps,\delta>0$, $f:\{-1,1\}^n\ra\{-1,1\}$, and $(x^i,f(x^i))_{i=1,\ldots,m}$ be a labeled random walk of length 
\[m\geq \frac{2N}{\eps^2}\ln\left(\frac{2N|\C|}{\delta}\right)\; ,\] 
where $N=\lceil n \ln(n|C|/\delta)\rceil$. Then, with probability at least $1-\delta$, for every $h\in \C$, 
\begin{equation}\label{eqn:convergence_all}
|\Delta(h,\S) - \Delta(h,f)|\leq \eps\; .
\end{equation}
\end{corollary}
\begin{proof}
Let $h\in\C$. Taking $g(x)=\frac 12 |h(x)-f(x)|$, we obtain $\Delta(h,\S)=\frac{1}{m}g(x)$ and $\Delta(h,f)=\E_x[g(x)]$, so that by Lemma~\ref{lem:concentration}, $|\Delta(h,\S)-\Delta(h,f)|\leq \eps$ with probability at least $1-\delta/|\C|$. Thus, with probability at least $1-\delta$, (\ref{eqn:convergence_all}) holds for {\em all} $h\in \C$.
\end{proof}

The following proposition shows that, similarly to the classical result by Angluin and Laird~\cite{AngLai88Learning} for distribution-free PAC-learning and the analogue by Kearns et al.~\cite{KeScSe94Toward} for agnostic PAC-learning, also in the random walk model agnostic learning is achieved by finding a hypothesis that minimizes the number of disagreements with a labeled random walk of sufficient length. 
\begin{proposition}\label{prop:consistent}
Let $\C=\C_n$ be a class of functions from $\{-1,1\}^n$ to $\{-1,1\}$. Let $\eps,\delta>0$, $f:\{-1,1\}^n\ra\{-1,1\}$, and $\S=(x^i,f(x^i))_{i=1,\ldots,m}$ be a labeled random walk of length $m\geq (8N/\eps^2)\ln(2N|\C|/\delta)$, where $N=\lceil n \ln(2n|\C|/\delta)\rceil$. Let $h_{\opt}\in\C$ minimize $\Delta(h,\S)$. Then, with probability at least $1-\delta$, 
$\Delta(h_{\opt},f)\leq \opt_{\C}(f) + \eps$. In particular, the required sample size is polynomial in $n$, $\log|\C|$, $1/\eps$, and $\log(1/\delta)$. 
\end{proposition}
\begin{proof}
By Corollary~\ref{cor:convergence_all}, $|\Delta(h,\S)-\Delta(h,f)|\leq \eps/2$ for all $h\in \C$. In particular, all functions $h\in \C$ with $\Delta(h,f) > \opt_\C(f)+\eps$ have 
$\Delta(h,\S) > \opt_\C(f)+\eps/2$, whereas all functions $h\in \C$ with $\Delta(h,f) = \opt_\C(f)$ have 
$\Delta(h,\S)\leq \opt_\C(f) + \eps/2$. Consequently, $\Delta(h_{\opt},\S)\leq \opt_\C(f) + \eps/2$, and thus $\Delta(h_{\opt},f)\leq \opt_C(f) + \eps$.
\end{proof}

\section{Agnostically Learning Juntas}\label{sec:juntas}

We start with our main technical lemma that shows that whenever there is a $k$-junta $g$ at distance $\Delta(f,g)$ to some function~$f$, then there is another $k$-junta $g'$ (in fact, a subfunction of $g$) at distance $\Delta(f,g)+\eps$ such that the relevant variables of $g'$ can be detected by finding all low-level Fourier coefficients that are of a certain minimum size.

\begin{lemma}\label{lem:spectrum}
Let $f:\{-1,1\}^n\ra\{-1,1\}$ be an arbitrary function and $g:\{-1,1\}^n\ra\{-1,1\}$ be a $k$-junta. 
Then, for every $\eps>0$, there exists a $k$-junta $g'$ such that 
$\langle f,g'\rangle\geq \langle f, g\rangle -\eps$ and 
for all relevant variables $x_i$ of $g'$, there exists $S\subseteq[n]$ with $|S|\leq k$, $i\in S$, and 
\begin{equation}\label{eqn:minhatf}
|\hat f(S)|\geq C\cdot 2^{-(k-1)/2}\cdot \eps\; ,
\end{equation}
where $C=1-1/\sqrt{2}\approx 0.293$. 
\end{lemma}
\begin{proof}
The proof is by induction on~$k$. For $k=0$, there is nothing to show since there are no relevant variables. For the induction step, let $k>0$. 
Assume that taking $g'$ to be $g$ does not satisfy the conclusion, i.e., for some relevant variable $x_i$ of $g$, $|\hat f(S)| < C2^{-(k-1)/2}\eps$ for all $S\subseteq[n]$ with $|S|\leq k$ and $i\in S$. Our goal is to show that in this case, either $g_{x_i=1}$ or $g_{x_i=-1}$ is well correlated with $f$ and thus asserts the existence of an appropriate ($k-1$)-junta~$g'$.

Let $\T=\{S\subseteq[n]\mid \hat g(S)\neq 0\}$. Then $|\T|\leq 2^k$. It follows that
\begin{eqnarray*}
\langle f,g\rangle & = & \sum_{S\in\T}\hat f(S)\hat g(S) \ = \ \sum_{S\in \T:i\in S}\hat f(S)\hat g(S) + \sum_{S\in\T:i\not\in S}\hat f(S)\hat g(S)\\
& \leq & \sum_{S\in \T:i\in S}|\hat g(S)|\cdot C\cdot 2^{-(k-1)/2}\cdot \eps + \sum_{S\in\T:i\not\in S}\hat f(S)\hat g(S)\\
& \leq & 2^{(k-1)/2}\cdot C\cdot 2^{-(k-1)/2}\eps + \sum_{S\in\T:i\not\in S}\hat f(S)\hat g(S)\ = \ C\cdot \eps + \sum_{S\in\T:i\not\in S}\hat f(S)\hat g(S)\ \; ,
\end{eqnarray*}
where the first equation is Plancherel's equation~(\ref{eqn:plancherel}) and the second inequality follows by Cauchy-Schwartz (note that $\hat g(S)$ is supported on at most $2^{k-1}$ sets $S$ with $i\in S$).
Consequently, 
\[\sum_{S\in\T:i\not\in S}\hat f(S)\hat g(S)\ \geq \ \langle f,g\rangle - C\cdot \eps \; .\]
Since for $S\subseteq[n]$, 
\[\left(\widehat{g_{x_i=1}}(S) + \widehat{g_{x_i=-1}}(S)\right)/2 = \begin{cases}
0 & \text{if }i\in S\\
\hat g(S) & \text{if }i\not\in S
\end{cases}\; ,\]
it follows that 
\[\langle f,g_{x_i=a}\rangle \ \geq \ \langle f,g\rangle - C\cdot \eps \] 
for $a=1$ or for $a=-1$. Now $g_{x_i=a}$ is a $(k-1)$-junta, so by induction hypothesis, there exists some $(k-1)$-junta $g'$ such that
\[\langle f,g' \rangle\geq \langle f,g_{x_i=a}\rangle - \eps/\sqrt{2} \geq \langle f,g\rangle - C\cdot \eps - \eps/\sqrt{2} = \langle f,g\rangle - \eps
\] 
and for all $x_i$ relevant to $g'$, there exists $S\subseteq[n]$ with $|S|\leq k-1$, $i\in S$, and 
\[
|\hat f(S)| \geq C\cdot 2^{-(k-2)/2}\cdot \eps/\sqrt{2} = C\cdot 2^{-(k-1)/2}\cdot \eps\; .\]
\end{proof}

One might wonder if for $f:\{-1,1\}^n\ra\{-1,1\}$ and a $k$-junta $g:\{-1,1\}^n\ra\{-1,1\}$, $\langle f,g\rangle\geq \eps$ does not imply that for every relevant variable $x_i$ of $g$, there exists $S\subseteq[n]$ with $|S|\leq k$, $i\in S$, such that~(\ref{eqn:minhatf}) holds. First of all, if $f(x)=x_1\wedge\ldots \wedge x_k$ (interpreting $-1$ as true and $+1$ as false), then for all $S\subseteq[n]$ with $S\neq \emptyset$, $|\hat f(S)|\leq 2^{-k+1}$. So taking $g=f$, the prior statement cannot hold. 

Still, one might at least hope for a similar statement with the right-hand side of~(\ref{eqn:minhatf}) replaced by something of the form $2^{-\poly(k)}\cdot \poly(\eps)$. However, if we take $f$ as above and $g(x)=x_2\wedge\ldots\wedge x_{k+1}$, then $\langle f,g\rangle = 1 - 2^{-k+1}$ but for all $S\subseteq[n]$ with $k+1\in S$, $\hat f(S)=0$ (since $x_{k+1}$ is not relevant to~$f$). 

Next, we need a tool for finding large low-degree Fourier coefficients of an arbitrary Boolean function, having access to a labeled random walk. Such an algorithm is said to perform the {\em Bounded Sieve} (see~\cite[Definition~3]{BMOS05Learning}). Bshouty et al.~\cite{BMOS05Learning} have shown that such an algorithm exists for the random walk model. More precisely, Theorems~7 and~9 in \cite{BMOS05Learning} imply:
\begin{theorem}[Bounded Sieve, \cite{BMOS05Learning}]
There is an algorithm $\mathrm{BoundedSieve}(f,\theta,\ell,\delta)$ that on input $\theta>0$, $\ell\in[n]$, and $\delta>0$, given access to $\RW(f)$ for some $f:\{-1,1\}^n\ra\{-1,1\}$, outputs a list of $S\subseteq[n]$ with $\hat f(S)^2\geq \theta/2$ such that with probability at least $1-\delta$, every $S\subseteq[n]$ with $|S|\leq \ell$ and $\hat f(S)^2\geq \theta$ appears in it. The algorithm runs in time $\poly(n,2^{\ell},1/\theta,\log(1/\delta))$, and the list contains at most $2/\theta$ sets~$S$.
\end{theorem}

For a sample $\S=(x^i,f(x^i))_{i=0,\ldots,m}$, a set $J\subseteq[n]$ of size $k$, and an assignment $\alpha\in\{-1,1\}^{|J|}$, let $s_\alpha^+=|\{i\in[m]\mid x^i|_J=\alpha\wedge f(x^i)=+1\}|$ and $s_\alpha^-=|\{i\in[m]\mid x^i|_J=\alpha\wedge f(x^i)=-1\}|$. Obviously, a $J$-junta~$h_J$ that best agrees with $f$ on the points in $\S$ is given by $h_J(x)=\sgn(s_{x|_J}^+ - s_{x|_J}^-)$. In other words, $h(x)$ takes on the value $a\in\{-1,1\}$ that is taken on by the majority of labels in the sub-cube that fixes the coordinates in $J$ to $\alpha$. This function is unique except for the choice of $h_J(x)$ at points $x$ with $s_{x_J}^+=s_{x_j}^-$. The function $h_J$ differs from the labels of $\S$ in $\err(J)=\sum_{\alpha\in\{-1,1\}^{|J|}} \err(\alpha)$ points, where $\err(\alpha)=\min\{s_\alpha^+,s_\alpha^-\}$. By Proposition~\ref{prop:consistent}, if $\S$ is sufficiently large,
then with high probability, the function $h_J$ approximately minimizes $\Delta(h,f)$ among all $J$-juntas~$h$.

We are now ready to show our main result:
\begin{theorem}[Restatement of Theorem~\ref{thm:main}]
The class of $k$-juntas $g:\{-1,1\}^n\ra\{-1,1\}$ is properly agnostically learnable with accuracy $\eps$ and confidence $1-\delta$ in the random walk model in time $\poly(n,2^{k^2},(1/\eps)^k,\log(1/\delta))$.
\end{theorem}
\begin{proof}

In the following, we show that Algorithm~\ref{alg:LearnJuntas} below
is an agnostic learning algorithm with the desired running time bound.

\begin{algorithm}[htb]
\caption{LearnJuntas}\label{alg:LearnJuntas}
\begin{algorithmic}[1]
 \STATE {\bf Input} $k,\eps,\delta$
 \STATE {\bf Access to} $\RW(f)$ for some $f:\{-1,1\}^n\ra\{-1,1\}$
 \STATE Run $\mathrm{BoundedSieve}(f,(1-1/\sqrt{2})^2\cdot 2^{-k+1}\cdot \eps^2, k, \delta/2)$ and let $\T$ be the returned list.
 \STATE Let $R=\bigcup\{S\mid S\in \T\}$.
 \STATE For all $J\subseteq R$ with $|J|=k$:
 \STATE \quad Compute $\err(J)$.
 \STATE {\bf Return} $h_{J_{\opt}}$ for some $J_{\opt}$ that minimizes $\err(J)$.
\end{algorithmic}
\end{algorithm}

\goodbreak

Denote the class of $n$-variate $k$-juntas by $\C$ and let $\gamma=\opt_\C(f)$. We prove that, with probability at least $1-\delta$, 
\[\Delta(h_{J_{\opt}},f)\leq \gamma+\eps\; .\]
Let $g\in\C$ with $\Delta(f,g)=\gamma$, so that $\langle f,g\rangle = 1-2\gamma$. By Lemma~\ref{lem:spectrum}, there exists $g'\in\C$ such that $\langle f,g'\rangle \geq 1-2\gamma-\eps$ (equivalently, $\Delta(f,g')\leq \gamma + \eps/2$) and for all relevant variables $x_i$ of $g'$, there exists $S\subseteq[n]$ with $|S|\leq k$, $i\in S$, and 
\[\hat f(S)^2\geq (1-1/\sqrt{2})^2\cdot 2^{-(k-1)}\cdot \eps^2 \; .\] Consequently, with probability at least $1-\delta/2$, the list $\T$ returned in Step~3 of the algorithm contains all of these sets~$S$, and thus $R$ contains all relevant variables of~$g'$. The Bounded Sieve subroutine runs in time $\poly(n,2^k,1/\eps,\log(1/\delta))$.

The set $J_{\opt}$ is chosen such that the corresponding $J_{\opt}$-junta $h_{J_{\opt}}$ minimizes the number of disagreements with the labels among all $k$-juntas with relevant variables in $R$. Denote the class of these juntas by $\C(R)$. 
Since $|\T|\leq 2\cdot(1-1/\sqrt{2})^{-2}2^{k-1}/\eps^2\leq 12\cdot 2^k / \eps^2$, we have $|R|\leq k|\T|\leq 12\cdot k\cdot 2^k/\eps^2$. Consequently, $R$ contains 
\[\binom{|R|}{k}\leq \left(\frac{e|R|}{k}\right)^k\leq \left(\frac{12\cdot 2^k}{\eps^2}\right)^k = \poly(2^{k^2},(1/\eps)^k)\]
subsets of size~$k$, and $\log |\C(R)|\leq \log\left(2^{2^k}\cdot \binom{|R|}{k}\right)=\poly(2^{k^2},(1/\eps)^k)$.

By Proposition~\ref{prop:consistent}, with probability at least $1-\delta/2$, 
\[\Delta(h_{J_{\opt}},f)\leq \opt_{\C(R)}(f) + \eps/2\; ,\]
provided that $\poly(n,\log|\C(R)|,1/\eps,\log(1/\delta))=\poly(n,2^{k^2},(1/\eps)^k,\log(1/\delta))$ examples are drawn.
Since $g'\in\C(R)$, we obtain
\[\Delta(h_{J_{\opt}},f)\leq \Delta(g',f) + \eps/2 \leq \gamma + \eps\; .\]
The total running time of the algorithm is polynomial in $n$, $2^{k^2}$, $(1/\eps)^k$, and $\log(1/\delta)$.
\end{proof}

\appendix

\section{Independence of Points in Random Walks}\label{sec:independence}

An {\em updating random walk} is a sequence $x^0,(x^1,i_1),(x^2,i_2),\ldots$, where $x^0$ is drawn uniformly at random, each $i_t\in[n]$ is a coordinate drawn uniformly at random, and $x^t$ is set to $x^{t-1}$ or to $x^{t-1}\odot e_{i_t}$, each with probability~$1/2$. We say that in step~$t$, coordinate $i_t$ is {\em updated}.

Given an {\em updating} random walk $x^0,(x^1,i_1),(x^2,i_2),\ldots$, all variables will with high probability be updated after $\ell=\Omega(n\log n)$ steps, so that in this case, $x^0$ and $x^\ell$ can be considered as independent uniformly distributed random variables. More formally, let $\mathcal X^\ell$ be the set of all updating random walks of length $\ell$, and let $\mathcal X^\ell_{\text{good}}$ be the set of updating random walks such that all variables have been updated (at least once) after $\ell$ steps. Then, conditional to the updating random walk belonging to $X^\ell_{\text{good}}$, $x^0$ and $x^\ell$ are independent (and uniformly distributed). 

Since the updating random walk model is only a technical utility, we would like to say similar things about the ``usual'' random walk model, so that we do not have to take care of going back and forth between the models in our analyses (although that would constitute a reasonable alternative).

We proceed as follows. Given a (non-updating) random walk $x^0,x^1,\ldots$, we perform an additional experiment to simulate an updating random walk (see also~\cite{BMOS05Learning}). We then {\em accept} the (original) random walk if the additional experiment leads to a {\em good} updating random walk. It will then follow that, conditional to the random walk being accepted, $x^0$ and $x^\ell$ are independent. Our algorithms will of course not perform this experiment. Instead, we will reason in the analyses that {\em if} we performed the additional experiment, {\em then} we would accept the given random walk with a certain (high) probability (taken over the draw of the random walk {\em and} the additional experiment), implying that certain points are independent.

Perform the following random experiment: Given a random walk $X$ of length~$\ell$, draw a sequence $F=(F_1,F_2,\ldots)$ of Bernoulli trials with $\Pr[F_j=1]=\Pr[F_j=0]=1/2$ for each $j$ until $F$ contains $\ell$ ones. (If this is not the case after, say, $L=\poly(\ell)$ steps, then reject $X$.) Otherwise, let $\ell'$ denote the length of $F$ and construct a sequence $I=(i_1,\ldots,i_{\ell'})$ of variable indices as follows. Denote by $j_1 < \ldots < j_\ell$ the $\ell$ positions in $F$ with $F_i=1$. 
For each $k\in[\ell]$, let $i_{j_k}=p_k$, where $p_k$ is the position in $X$ that is flipped in the $k$th step. 
For each $j\in[\ell']\setminus\{j_1,\ldots,j_\ell\}$, independently draw an index $i_j\in[n]$ with uniform probability. Accept $X$ if $\{i_1,\ldots,i_{\ell'}\}=[n]$, otherwise reject~$X$.

\begin{lemma}[Formal restatement of Lemma~\ref{lem:independence}]\label{lem:independence_formal}
Let $X=(x^0,\ldots,x^\ell)$ be a random walk of length $\ell\geq n\ln(2n/\delta)$ and perform the experiment above. Then $X$ is accepted with probability at least $1-\delta$. Moreover, conditional to $X$ being accepted, the random variables $x^0$ and $x^\ell$ are independent and uniformly distributed.
\end{lemma}
\begin{proof}
First, by choosing $L$ appropriately, we can ensure with probability at least $1-\delta/2$ that $F$ contains at least $\ell$ ones.
By construction, the sequence ${x'}^0,({x'}^1,i_1),(x'^2,i_2),\ldots,(x'_{\ell'},i_{\ell'})$ with $x'^0=x^0$, $x'^{j_k}=x_k$ for $k\in[\ell]$, and $x'^j=x'^{j-1}$ for $j\in[\ell']\setminus \{j_1,\ldots,j_\ell\}$ is distributed as an updating random walk of length $\ell'$. Note that unlike in the original updating random walk model, we determine the sequence $F$ of updating outcomes {\em before} we determine the positions to be updated. Moreover, the choice of the coordinates to be updated in the positions where $F_i=1$ is incorporated in the draw of the original walk. The subsequence $x'^0,x'^{j_1},x'^{j_2},\ldots,x'^{j_{\ell}}$ is equal to the original walk $x^0,x^1,x^1,x^2,\ldots, x^\ell$. 

The probability that $\{i_1,\ldots,i_{\ell'}\}\subsetneq[n]$ is at most 
\[n\cdot (1-1/n)^{\ell'}\leq n\cdot (1-1/n)^{\ell}\leq \delta/2\] 
since $\ell\geq n\ln (2n/\delta)$.
Consequently, with total probability at least $1-\delta$, the random walk is accepted. In this case, every coordinate has eventually been updated after the $\ell'$ steps of the updating random walk. Thus, for each coordinate~$j$, of $x_j^\ell=x'^{\ell'}_j$ is independent of $x_j^0=x'^0_j$, i.e., $x^0$ and $x^\ell$ are independent and uniformly distributed (conditional to $X$ being accepted).
\end{proof}

\section{An Elementary Proof of Lemma~\ref{lem:concentration}}\label{sec:proof_concentration}

To estimate the convergence rate of empirical averages to their expectations, we need the following standard Chernoff-Hoeffding bound~\cite{Hoe63Probability}: For a sequence of independent identically distributed random variables $X_1,\ldots,X_m$ with $\E[X_i]=\mu$ that take values in $[-1,1]$,
\begin{equation}\label{eqn:chernoff}
\Pr\left[\left|\frac{1}{m}\sum_{i=1}^m X_i - \mu\right|\right]\leq 2 e^{-\eps^2 m /2}\; .
\end{equation}

\begin{proof}[Proof of Lemma~\ref{lem:concentration}]
For each $j\in\{0,\ldots,N-1\}$, the points $x^{iN+j}$, $i\in\{0,\ldots,m/N-1\}$, are with probability at least $1-(m/N - 1)\delta$ pairwise independent by Lemma~\ref{lem:independence}. In this case, the values $f(x^{iN+j})$, $0\leq i\leq m/N - 1$, are independent and identically distributed samples of the random variable $f(x)$ with $x\in\{-1,1\}^n$ uniformly distributed. By the Hoeffding bound,
\[\Pr\left[\left|\frac{N}{m}\sum_{i=0}^{m/N-1}f(x^{iN+j}) - \E_{x}[f(x)]\right| > \eps\right] \leq 2\exp(-m\eps^2/(2N))\]
Thus, the probability that $\left|(N/m)\sum_{i=0}^{m/N-1}f(x^{iN+j}) - \E_{x}[f(x)]\right|>\eps$ for {\em some} $j\in\{0,\ldots,N-1\}$ is at most $2N\exp(-m\eps^2/(2N))$.
Finally, we have
\begin{eqnarray*}
\left|\frac{1}{m}\sum_{i=0}^m f(x^i) - \E_{x}[f(x)]\right| & = & \frac{1}{m}\left|\sum_{j=0}^N\left(\sum_{i=0}^{m/N-1} f(x^{iN+j}) - \frac{m}{N}\E_{x}[f(x)]\right)\right|\\
& \leq & \frac{1}{m}\sum_{j=1}^N \frac{m}{N}\eps = \eps
\end{eqnarray*}
with probability at least $1-2N\exp(-m\eps^2/(2N))\geq 1-\delta$. 
\end{proof}

\end{document}